\documentclass[twoside,11pt]{article}

\usepackage{jmlr2e}

\usepackage[utf8]{inputenc}
\usepackage[T1]{fontenc}    

\usepackage{natbib}
\usepackage{amsmath}
\usepackage{amssymb}
\usepackage{bm}
\usepackage{algorithm}
\usepackage[noend]{algorithmic}
\usepackage{paralist}
\usepackage{wrapfig}
\usepackage{floatflt}
\usepackage{mathtools}
\usepackage{mathabx}
\usepackage{xspace}
\usepackage[toc,page]{appendix}

\usepackage[capitalize]{cleveref}
\crefname{assumption}{Assumption}{Assumption}
\usepackage{thmtools}
\usepackage{thm-restate}

\declaretheorem[name=Lemma,refname={Lemma,Lemmas},Refname={Lemma,Lemmas},sibling=theorem]{lemma}

\declaretheorem[name=Corollary,refname={Corollary,Corollaries},Refname={Corollary,Corollaries},sibling=theorem]{corollary}

\declaretheorem[name=Assumption,refname={Assumption,Assumptions},Refname={Assumption,Assumptions}]{assumption}

\declaretheorem[name=Proposition,refname={Proposition,Propositions},Refname={Proposition,Propositions},sibling=theorem]{proposition}

\declaretheorem[name=Definition,refname={Definition,Definitions},Refname={Definition,Definitions}]{definition}

\Crefname{question}{Question}{Questions}
\creflabelformat{question}{(#2#1#3)}

\Crefname{problem}{Problem}{Problems}
\creflabelformat{problem}{(#2#1#3)}

\Crefname{equation}{Equation}{Equations}
\creflabelformat{equation}{(#2#1#3)}

\crefname{iCondition}{Condition}{Conditions}
\creflabelformat{iCondition}{(#2#1#3)}
\crefrangelabelformat{iCondition}{(#3#1#4) to (#5#2#6)}

\Crefname{item}{}{}
\creflabelformat{item}{(#2#1#3)}
\crefrangelabelformat{item}{(#3#1#4) to (#5#2#6)}

\renewcommand{\P}{\mathbb{P}}
\newcommand{\E}{\mathbb{E}}
\newcommand{\EE}[1]{\E\left[#1\right]}
\newcommand{\Prob}[1]{\P\left(#1\right)}
\newcommand{\R}{\mathbb{R}}
\newcommand{\N}{\mathbb{N}}

\renewcommand{\S}{\mathcal{S}}

\DeclareMathOperator*{\argmax}{argmax}

\usepackage{todonotes}

\makeatletter
\def\moverlay{\mathpalette\mov@rlay}
\def\mov@rlay#1#2{\leavevmode\vtop{%
   \baselineskip\z@skip \lineskiplimit-\maxdimen
   \ialign{\hfil$\m@th#1##$\hfil\cr#2\crcr}}}
\newcommand{\charfusion}[3][\mathord]{
    #1{\ifx#1\mathop\vphantom{#2}\fi
        \mathpalette\mov@rlay{#2\cr#3}
      }
    \ifx#1\mathop\expandafter\displaylimits\fi}
\makeatother

\newlength{\algwidth}
\newlength{\plotwidth}
\newcommand{\KL}{\operatorname{D}}

\newcommand{\one}[1]{\mathbb{I}\left\{#1\right\}}

\newcommand{\hmu}{\widehat{\mu}}
\newcommand{\natu}{\mathbb{N}}

\DeclareMathOperator{\join}{join}
\DeclareMathOperator{\MinMax}{MinMax}
\DeclareMathOperator{\Hsucc}{H_{\text{succ}}}
\newcommand{\LUCBMinMax}{{\sc LUCBMinMax}\xspace}

\DeclareMathOperator{\Span}{span} 

\newcommand{\cB}{\mathcal{B}}
\newcommand{\cH}{\mathcal{H}}

\jmlrheading{.}{...}{...--...}{.../...}{.../...}{ajallooe17a}{Ruitong Huang and Mohammad M. Ajallooeian and Csaba Szepesv\'{a}ri and Martin M\"{u}ller}

\ShortHeadings{Structured Best Arm Identification}{Huang and Ajallooeian and Szepesv\'{a}ri and M\"{u}ller}
\firstpageno{1}

\begin{document}
	
	\title{Structured Best Arm Identification with Fixed Confidence}
	
	\author{\name Ruitong Huang \email ruitong@ualberta.ca
		\AND
		\name Mohammad M.\ Ajallooeian \email ajallooe@ualberta.ca
		\AND
		\name Csaba Szepesv\'{a}ri \email szepesva@ualberta.ca
		\AND
		\name Martin M\"{u}ller \email mmueller@ualberta.ca \\
		\addr Department of Computing Science\\
		University of Alberta\\
		Edmonton, AB T6G 2E8, Canada}
	
	\editor{... and ...}
	
	\maketitle
	
	\begin{abstract}
		We study the problem of identifying the best action among a set of possible options 
		when the value of each action is given by a mapping from a number of noisy micro-observables in the so-called fixed confidence setting.
		Our main motivation is the application to the minimax game search, which has been a major topic of interest in artificial intelligence.
		In this paper we introduce an abstract setting to clearly describe the essential properties of the problem.
		While previous work only considered a two-move game tree search problem,
		our abstract setting can be applied to the general minimax games where the depth can be non-uniform and arbitrary, and transpositions are allowed.
		We introduce a new algorithm (LUCB-micro) for the abstract setting, and give its lower and upper sample complexity results. 
		Our bounds recover some previous results, which were only available in more limited settings, while they also shed further light on how the structure of minimax problems influence sample complexity.
	\end{abstract}
	
	\section{Introduction}
	Motivated by the problem of finding the optimal move in minimax tree search with noisy leaf evaluations, we introduce best arm identification problems with structured payoffs and micro-observables.
	In these problems, the learner's goal is to find the best arm when the payoff of each arm is a fixed and known function of a set of unknown values.
	In each round, the learner can choose one of the micro-observables to make a noisy measurement 
	(i.e., the learner can obtain a ``micro-observation'').
	We study these problems in the so-called fixed confidence setting.
	
	A special case of this problem is the standard best arm identification, which has seen a flurry of activity during the last decade, e.g., \citep{even2006action,audibert2010best,gabillon2012best,kalyanakrishnan2012pac,karnin2013almost,jamieson2014lil,chen2015optimal}. 
	Recently, \citet{garivier2016maximin} considered the motivating problem mentioned above.
	However, they only considered the simplest (non-trivial) instance when two players alternate for a single round.
	One of their main observations is that such two-move problems can be solved more efficiently than if 
	one considers the problem as an instance of a nested best arm identification problem. 
	They proposed two algorithms, one for the fixed confidence setting, the other for the (asymptotic) vanishing confidence setting
	and provided upper bounds. An implicit (optimization-based) lower bound was also briefly sketched, together with a 
	plan to derive an algorithm that matches it in the vanishing confidence setting.
	
	Our main interest in this paper is to see whether the ideas of \citet{garivier2016maximin} extend to more general settings,
	such as when the depth can be non-uniform and is in particular not limited to two, or when the move histories can lead to shared states (that is, in the language of adversarial search we allow ``transpositions'').
	While considering these extensions, we found it cleaner to introduce the abstract setting mentioned below (\cref{sec:prob}). 
	The motivation here is to clearly delineate the crucial properties of the problem that our results use. 
	For the general structured setting, in \cref{sec:lb}
	we prove an instance dependent lower bound along the lines of \citet{AuCBFrSch02} or \citet{GKL16}
	(a mild novelty is the way our proof deals with the technical issue that best arm identification algorithms ideally stop and hence their behavior is undefined after the random stopping time).
	This is then specialized to the minimax game search setting (\cref{sec:lb_minmax}), where we show the crucial
	role of what we call proof sets, which are somewhat reminiscent of the so-called conspiracy sets from adversarial search
	\citep{mcallester1988conspiracy}.
	Our lower bound matches that of \citet{garivier2016maximin} in the case of two-move alternating problems.
	Considering again the abstract setting, we propose a new algorithm, which we call LUCB-micro (\cref{sec:ub}),
	and which can be considered as a natural generalization of 
	Maximin-LUCB of  \citet{garivier2016maximin} (with some minor differences).
	Under a regularity assumption on the payoff maps, 
	we prove that the algorithm meets the risk-requirement.
	We also provide a high-probability, instance-dependent 
	upper bound on algorithm's sample complexity (i.e., on the number of observations the algorithm takes).
	As we discuss, while this bound meets the general characteristics of existing bounds, it fails to reproduce 
	the corresponding result of \citet{garivier2016maximin}.
	To the best of authors' knowledge, the only comparable algorithm to study best arm identification in a full-length minimax tree search setting (which was the motivating example of our work) is FindTopWinner by \citet{teraoka2014efficient}.
	This algorithm is a round-based elimination based algorithm with additional pruning steps that come from the tree structure.
	When we specialize our framework to the minimax game scenario (and implement other necessary changes to put our work into their ($\epsilon,\delta)$-PAC setting), our upper bound is a strict improvement of theirs, e.g., in the number of samples related to the \emph{near-optimal} micro-observables (leaves of the minimax game tree).
	Next, we consider the minimax setting (\cref{sec:minimaxub}). First, we show that the regularity assumptions
	made for the abstract setting are met in this case. We also show how to efficiently compute the choices that
	LUCB-micro makes using a ``min-max'' algorithm. Finally, we strengthen our previous result so that it is able
	to reproduce the mentioned result of \citet{garivier2016maximin}.

	\subsection{Notation}
	We use $\N = \{1,2,\dots\}$ to denote the set of positive integers, while we let $\R$ denote the set of reals.
	For a positive integer $k\in \N$, we let $[k] = \{ 1, \dots, k\}$.
	For a vector $v \in \R^d$, we denote its $i$-th element by $v_i$; though occasionally we will also use $v(i)$ for the same
	purpose, i.e., we identify  $\R^d$ and $\{f: f:[d]\to \R\}$ in the obvious way.
	We let $|v|$ denote the vector defined by $|v| = (|v_i|)_{i\in [d]}$.
	For two vectors $u,v\in \R^d$, we define $u\le v$ if and only if $u_i \le v_i$ for all $1\le i\le d$.
	Further, we write $u<v$ when $u\ne v$ and $u\le v$.
	For $B\subset [d]$, we write $u|_B$ to denote the $|B|$-dimensional vector obtained from
	restricting $u$ to components with index in $B$: $u|_B =(u_i)_{i\in B}$.
	\newcommand{\allone}{\mathbf{1}}
	We use $\allone_d$ to denote the $d$ dimensional vector whose components are all equal to one.
	For a nonempty set $B$, we also use $\allone_B$ to denote the $|B|$-dimensional all-one vector.
	We let $B^c = \{ i\in [d]\,:\, i\not\in B \}$ to denote the complementer of $B$ (when $B^c$ is used, the base set 
	that the complementer is taken for should be clear from the context).
	The indicator function will be denoted by $\one{\cdot}$.
	We will use $a \wedge b = \min(a,b)$ and $a \vee b = \max(a,b)$.
	For $A\subset \R$, $\bar A$ denotes its topological closure, while $A^\circ$ denotes its interior. 
	Given a real value $a\in\R$, $a_+ = a\vee 0$ and $a_- = -(a\wedge 0)$. 
	For a sequence $(m_0,\dots,m_i)$ of some values and some other value $m$, 
	we define $\join(h,m) = (m_0,\dots,m_i,m)$.
	
	\section{Problem setup}
	\label{sec:prob}
	\newcommand{\defeq}{\doteq}
	Fix two positive integers, $L$ and $K$.
	A problem instance of a \emph{structured $K$-armed best arm identification instance with $L$ micro-observations} is defined by a tuple $(f,P)$, where $f:\R^L \to \R^K$ and $P=(P_1,\dots,P_L)$ is an $L$-tuple of distributions over the reals.
	We let  $\mu_i = \int x dP_i(x)$ denote the mean of distribution $P_i$.
	We shall denote the component functions of $f$ by $f_1,
	\dots,f_K$: $f(\mu) = (f_1(\mu),\dots,f_K(\mu))$.
	The value $f_i(\mu)$ is interpreted as the payoff of arm $i$ and we call $f$ the \emph{reward map}.
	The goal of the learner is to identify the arm with the highest payoff.
	It is assumed that the arm with the highest payoff is unique.
	The learner knows $f$, while is unaware of $P$, and, in particular, unaware of $\mu$.
	To gain information about $\mu$, the learner can query the distributions in discrete rounds indexed by $t=1,2,\dots$, in a sequential fashion.
	The learner is also given $\delta\in (0,1)$, a \emph{risk parameter} (also known as a confidence parameter).
	The goal of the learner is to identify the arm with the highest payoff using the least number of observations while keeping the probability of making a mistake below the specified risk level.
	\begin{wrapfigure}[16]{r}{0.5\textwidth}
		\vspace{-0.35cm}
		\begin{minipage}{.49\textwidth}
			\hrulefill
			\begin{figure}[H]
				\begin{algorithmic}
					\STATE \textbf{Input}: $\delta\in (0,1),f=(f_1,\dots,f_K)$
					\FOR{$t=1,2,\ldots$}
					\STATE Choose $I_t\in [L]$
					\STATE Observe $Y_t \sim P_{I_t}(\cdot)$
					\IF{Stop()}
					\STATE Choose $J\in [K]$, candidate optimal arm index
					\STATE $T\leftarrow t$
					\RETURN $(T,J)$
					\ENDIF
					\ENDFOR
					\STATE Admissibility: $\Prob{ J\ne \argmax_i f_i(\mu) }\le \delta$,
					$\Prob{T<\infty}=1$.
				\end{algorithmic}
				\vspace*{-.1in}
				\noindent\makebox[\linewidth]{\rule{\textwidth}{0.4pt}}%
				\vspace*{-.2in}
				\caption{Interaction of a learner and a problem instance $(f,P)$.
					The components of $\mu$ are $\mu_i =
					\int x dP_i(x)$, $i\in [L]$, and $f$ maps $\R^L$ to $\R^K$.
				}%
				\label{fig:protocol}
			\end{figure}%
		\end{minipage}%
	\end{wrapfigure}
	A learner is \emph{admissible} for a given set $\S$ of problem instances if
	{\em (i)} for any instance from $\S$, the probability of the learner misidentifying the optimal arm in the instance is below the given fixed risk factor $\delta$;
	and
	{\em (ii)} the learner stops with probability one on any instance from $\S$.
	The interaction of a learner and a problem instance is shown on \cref{fig:protocol}.
	
	\paragraph{Minimax games}
	As a motivating example, consider the problem of finding the optimal move for the first player in a finite two-player minimax game. 
	The game is finite because the game finishes in finitely many steps (by reaching
	one of the $L$ possible terminating states).
	The first player has $K$ moves.
	The value of each move  is a function of the values $\mu\in \R^L$ of the $L$ possible terminating states.
	
	\newcommand{\pmset}{\{-1,+1\}}
	Formally, such a minimax game is described by $G=(M,H,p,\tau)$, where $M$ is a non-empty finite set of possible moves, $H\subset \cup_{n\ge 0} M^n$ is a finite set of (feasible) histories of moves, the function $p: H \to \pmset$ determines, for each feasible history, the identity of the player on turn, and $\tau$ is a surjection that maps a subset $H_{\max}\subset H$ of histories, the set of maximal histories in $H$, to $[L]$ (in particular, note that $\tau$ may map multiple maximal histories to the same terminating state).
	An element $h$ of $H$ is maximal in $H$ if it is not the prefix of any other history $h'\in H$, or, in other words, if it has no continuation in $H$.
	The set $H$ has the property that if $h\in H$ then every prefix of $h$ with positive length is also is in $H$.
	The first player's moves are given by the histories in $H$ that have a unit length.
	To minimize clutter, without the loss of generality (WLOG), we identify this set with $[K]$.
	
	The function $f=(f_1,\dots,f_K)$ underlying $G$ gives the payoffs of the first player.
	To define $f$ we use the auxiliary function $V(\cdot,\mu):H \to \R$ 
	that evaluates any particular history given the values $\mu$ assigned to terminal states.
	Given $V$, we define $f_k(\mu) =V((k),\mu)$ for any $k\in [K]$.
	It remains to define $V$:
	For $h\in H_{\max}$, $V(h,\mu) = \mu_{\tau(h)}$.
	For any other feasible history $h\in H$, 
	$V(h,\mu) = p(h) \max\{ p(h) V(h',\mu)\,:\, h' \in \Hsucc(h)  \}$,
	where $\Hsucc(h) = \{ \join(h,m) \,:\, m\in M \} \cap H$ is the set of 
	\emph{immediate successors} of $h$ in $H$.
	Thus, when $p(h)=1$, $V(h,\mu)$ is the maximum of the values associated with
	the immediate successors of $h$, while when $p(h)=-1$, $V(h,\mu)$ is the minimum of these values.
	We define $m(h,\mu)$ as the move $m$ defining the optimal immediate successor of $h$ given $\mu$.
	Note that all many of the defined functions depend on $H$, but the dependence is suppressed, as we will keep $H$ fixed.
	One natural problem that fits our setting is a (small) game
	when the payoffs at the terminating states of a game are themselves randomized (e.g., at the end of a game some random hidden information such as face down cards can decide the value of the final state).
	As explained by \citet{garivier2016maximin}, the setting may also shed light on how to design better
	Monte-Carlo Tree Search (MCTS) algorithms, which is a relatively novel class of search algorithms
	that proved to be highly successful in recent years
	\citep[e.g.,][]{Gelly12MCTSReview,Si16:Nature}.
	
	\section{Lower bound: General setting}
	\label{sec:lb}
	\newcommand{\Snorm}{\S^{\textrm{norm}}}
	\newcommand{\PmuA}{\P_{\mu,A}}
	\newcommand{\EmuA}{\E_{\mu,A}}
	\newcommand{\PmupA}{\P_{\mu',A}}
	\newcommand{\EmupA}{\E_{\mu',A}}
	\newcommand{\cF}{\mathcal{F}}
	In this section we will prove a lower bound for the case of a fixed map $f$ and when the set of instances is the set all normal distributions with unit variance.
	We denote the corresponding set of instances by $\Snorm_{f}$.
	Our results can be easily extended to the case of other sufficiently-rich family of distributions.
	
	For the next result, assume without loss of generality that $f_1(\mu)>f_2(\mu)\ge \dots \ge f_K(\mu)$.
	Fix a learner (policy) $A$, which maps histories to actions.
	For simplicity, we assume that $A$ is deterministic (the extension to randomized algorithms is standard).
	Let $\Omega = ([L]\times \R)^\N$ be the set of (infinite) sequences of observable-index and observation pairs so that for any $\omega=(i_1,y_1,i_2,y_2,\dots)\in \Omega$, $t\ge 1$, $I_t(\omega) = i_t$ and $Y_t(\omega) = y_t$.
	We equip $\Omega$ with the associated Lebesgue $\sigma$-algebra $\cF$.
	For an infinite sequence $\omega = (i_1,y_1,i_2,y_2,\dots)\in \Omega$, we let $T(\omega)\in \N \cup \{\infty\}$ be the round index when the algorithm stops  (we let $T(\omega)=\infty$ if the algorithm never stops on $\omega$).
	Thus, $T:\Omega \to \N \cup \{\infty\}$.
	Similarly, define $J:\Omega \to [K+1]$ to be the choice of the algorithm when it stops, where we define $J(\omega)=K+1$ in case $T(\omega)=\infty$.
	
	The interaction of a problem instance (uniquely determined by $\mu$) and the learner (uniquely determined by the associated policy $A$) induces a unique distribution $\PmuA$ over the measurable space $(\Omega,\cF)$, where we agree that in rounds with index $t=T+1,T+2,\dots$, we specify that the algorithm chooses arm $1$, while the observation distributions are modified so that the observation is deterministically set to zero.
	We will also use $\EmuA$ to denote the expectation operator corresponding to $\PmuA$.
	
	To appease the prudent reader, let us note that our statements will always be concerned with events that are subsets of the event $\{T<\infty\}$ and as such they are not effected by how we specify the ``choices'' of the algorithm and the ``responses'' of the environment for $t>T$.
	Take, as an example, the expected number of steps that $A$ takes in an environment $\mu$, $\EmuA[T]$, which we bound below.
	Since we bound this only in the case when the algorithm $A$ is admissible, which implies that $\PmuA(T<\infty)=1$, we have $\EmuA[T] = \EmuA[T \,\one {T<\infty} ]$.
	which shows that the behavior of $\PmuA$ outside of $\PmuA$ outside of $\{T<\infty\}$ is immaterial for this statement.
	The choices we made for $t>T$ (for the algorithm and the environment) will however be significant in that they simplify a key technical result.
	
	To state our result, we need to introduce the set of \emph{significant departures}, $D_\mu\subset \R^L$, from $\mu$.
	This set contains all vectors $\Delta$ such that the best arm under $\mu+\Delta$ is not arm $1$.
	Formally,
	\begin{align}
	D_\mu =
	\{\Delta\in \R^L\,:\, f_1(\mu+\Delta)\le \max_{i>1} f_i(\mu+\Delta) \}\,.
	\end{align}
	
	\begin{restatable}[Lower bound]{theorem}{thmLB}
		\label{thm:LB}
		Fix a risk parameter $\delta\in (0,1)$.
		Assume that $A$ is admissible over the instance set $\Snorm_f$ at the risk level $\delta$.
		Define
		\begin{equation}
		\label{eq:LB_tau}
		\tau^*(\mu) = \min \left\{\sum_{i=1}^L n(i)\,:\, \inf_{\Delta\in D_{\mu}} \sum_{i=1}^L n(i) \Delta_i^2\ge 2\log(1/(4\delta)),\, n(1),\dots,n(L)\ge 0\, \right\}\,.
		\end{equation}
		Then, $\EmuA[T]\ge \tau^*(\mu)$.
	\end{restatable}
	The proof can be shown to reproduce the result of \citet{GaKa16} (see page 6 of their paper) 
	when the setting is best arm identification.
	The proof uses standard steps \citep[e.g.,][]{AuCBFrSch02,KauCaGa16} and one of its main merit is its simplicity.
	In particular, it relies on two information theoretical results; a high-probability Pinsker inequality (Lemma 2.6 from \citep{Tsy08:NonpBook}) and a standard decomposition of divergences. 
	The proof is given in \cref{sec:lbproof} (all proofs omitted from the main body can be found in the appendix).
	
	\begin{remark}[Minimal significant departures ($D^{\min}_\mu$)]
		From the set of significant departures one can remove all vectors $d$
		that are componentwise dominating in absolute value some
		other significant departure $\Delta\in D_\mu$ without effecting the lower bound.
		To see this, write the lower bound as $\min\{ \sum_i n(i)\,: n \in \cap_{\Delta\in D_\mu} \Phi(\Delta) \}$,
		where $\Phi(\Delta) = \{ n\in [0,\infty)^L\,:\, \sum_i n(i) \Delta_i^2 \ge 2 \log(1/(4\delta))\}$.
		Then, if $d,\Delta\in D_\mu$ are such that $|\Delta|\le |d|$
		then $\Phi(\Delta)\subset \Phi(d)$.
		Hence, $\cap_{\Delta\in D_\mu} \Phi(\Delta) = \cap_{\Delta\in D^{\min}_\mu} \Phi(\Delta)$
		where $D^{\min}_\mu = \{ d\in D_\mu\,: \, \not\exists \Delta\in D_\mu \text{ s.t. } |\Delta| < |d| \}$.
	\end{remark}
	
	\section{Lower bound for minimax games}
	\label{sec:lb_minmax}
	\newcommand{\tH}{\widetilde{H}}
	In this section we prove a corollary of the general lower bound of the previous section in the context of minimax games;
	the question being what role the structure of a game plays in the lower bound.
	For this section fix a minimax game structure 
	$G = (M,H,p,\tau)$ (cf. \cref{sec:prob}).
	We first need some definitions:
	
	\begin{definition}[Proof sets]
		Take a minimax game structure $G = (M,H,p,\tau)$ with $K$ first moves and $L$ terminal states.
		Take  $j\in [K]$.
		A set $B\subset [L]$ is said to be sufficient for proving upper bounds on the value of move $j$ if
		for any $\mu\in \R^L$ and $\theta\in \R$,  $\mu|_B = \theta \allone_B$ implies $f_j(\mu) \le \theta$.
		Symmetrically,
		a set $B\subset [L]$ is said to be sufficient for proving lower bounds on the value of move $j$ if
		for any $\mu\in \R^L$ and $\theta\in \R$, $\mu|_B = \theta \allone_B$ implies $f_j(\mu) \ge \theta$.
	\end{definition}
	We will call the sets satisfying the above definition upper (resp., lower) proof sets, denoted by $\cB_j^+$ (resp., $\cB_j^-$ ). 
	Proof sets are closely related to conspiracy sets 
	\citep{mcallester1988conspiracy}, forming the basis of ``proof number of search''
	\citep{allis1994searching,kishimoto2012game}.
	In a minimax game tree, a conspiracy set of a node (say, $v$) is the set of leaves 
	that must change their evaluation value to cause a change in the minimax value of that node $v$.
	Proof sets are also related to cuts in $\alpha$--$\beta$ search \citep{russell2010artificial}. 
	
	One can obtain minimal upper proof sets that belong to $\cB_j^+$ in the following way:
	Let $H_j$ denote the set of histories that start with move $j$.
	Consider a non-empty $\tH \subset H_j$ that satisfies the following properties:
	{\em (i)} if $h\in \tH$ and $p(h)=-1$ (minimizing turn) then $|\Hsucc(h) \cap \tH| = 1$;
	{\em (ii)} if $h\in \tH$ and $p(h)=1$ (maximizing turn) then $\Hsucc(h) \subset \tH$.
	Call the set of $\tH$ that can be obtained this way $\cH_j^+$. 
	From the construction of $\tH$ we immediately get the following proposition:
	\begin{restatable}{proposition}{propLB1}
		\label{prop:LB1}
		Take any $\tH\in \cH_j^+$ as above. Then, $\tau( \tH \cap H_{\max} ) \in \cB_j^+$. 
	\end{restatable}
	A similar construction and statement applies in the case of $\cB_j^-$, resulting in the set $\cH_j^-$.
	Our next result will imply that the lower bound is achieved by considering departures of a special form, related to proof sets:
	\begin{restatable}[Minimal significant departures for minimax games]{proposition}{propMiniSigDepart}
		\label{prop:MiniSigDepart}
		Assume WLOG that $f_1(\mu)>\max_{j>1} f_j(\mu)$.
		Let
		\begin{align*}
		S
		& =\Bigl\{\,\Delta\in\R^L\,:\,
		\exists 1<j\le K\,, \theta\in [f_j(\mu),f_1(\mu)], 
		B\in \cB_1^+, B'\in \cB_j^-  \text{ s.t. }\\
		& \qquad \qquad\qquad  
		\Delta_i=-(\mu_i-\theta)_+, \, \forall i\in B\setminus B';\, 
		\Delta_i=(\mu_i-\theta)_-, \, \forall i\in B'\setminus B; \\
		& \qquad \qquad\qquad \Delta_i=\theta-\mu_i, \, \forall i\in B'\cap B;\, \Delta_i=0, \, \forall i\in (B\cup B')^c \,
		\Bigr\}\,.
		\end{align*}
		Then, $D_\mu^{\min} \subset S \subset D_\mu$.
	\end{restatable}
	Note that the second inclusion shows that replacing $D_\mu$ by $S$ in the definition of $\tau^*(\mu)$ 
	would only decrease the value of $\tau^*(\mu)$, while the first inclusion shows that 
	the value actually does not change. 
	The following lemma, characterizing minimal departures, is essential for our proof of \cref{prop:MiniSigDepart}:
	\begin{restatable}{lemma}{lemMMLBRev}
		\label{lem:MMLBRev}
		Take any $\mu\in \R^L$, $d\in D_\mu^{\min}$ and assume WLOG that $f_1(\mu) > \max_{j>1} f_j(\mu)$. 
		Then, there exist $B\in \cB_1^+$, $j\in \{2,\dots,K\}$ and $B' \in \cB_j^-$ such that 
		\begin{enumerate}[(i)]
			\item \label{lem:MMLBRev:1} $\max\{(\mu+d)_i\,:\, i\in B\} = f_1(\mu+d) = f_j(\mu+d) = \min\{(\mu+d)_i\,:\, i\in B'\}$;
			\item \label{lem:MMLBRev:2} $d_i\le 0$ if $i\in B\setminus B'$; $d_i\ge 0$ if $i\in B'\setminus B$;
			\item \label{lem:MMLBRev:3}  $\forall i \in B\cup B'$, either $(\mu+d)_i = f_1(\mu+d)=f_j(\mu+d)$ or $d_i = 0$.
		\end{enumerate}
	\end{restatable}
	\cref{prop:MiniSigDepart} implies the following:
	\begin{corollary}
		Let $\mu$ be a valuation and assume WLOG that $f_1(\mu) >  \max_{j>1} f_j(\mu)$. 
		Let $\cB_j = \{(B,B')\,:\, B\in \cB_1^+, B'\in  \cB_j^-\}$. Then,
		\begin{align*}
		\tau^*(\mu) & = \min_{n\in [0,\infty)^L} 
		\Bigl\{
		\sum_i n(i)\,:\,  \min_{1<j\le K, \theta \in [f_j(\mu),f_1(\mu)], (B,B')\in \cB_j}
		\sum_{i\in B\setminus B' } n(i) (\mu_i-\theta)_+^2+\sum_{i\in B'\setminus B} n(i) (\mu_i-\theta)_-^2 \\
		&	\qquad \qquad \qquad \qquad +\sum_{i\in B\cap B' } n(i) (\theta-\mu_i)^2
		\ge 2 \log(\tfrac1{4\delta}) 
		\Bigr\}\,.
		\end{align*}
		Hence, for any algorithm $A$ admissible over the instance set $\Snorm_f$ at the risk level $\delta$, $\EmuA[T]$ is at least as large than the right-hand side of the above display.
	\end{corollary}
	
	\section{Upper bound}
	\label{sec:ub}
	
	In this section we propose an algorithm generalizing
	the LUCB algorithm of \citet{kalyanakrishnan2012pac} and prove a theoretical guarantee
	for the proposed algorithm's sample complexity 
	under some (mild) assumptions on the structure of the reward mapping $f$.
	Our result is inspired and extends the results of \citet{garivier2016maximin} (who also started from the LUCB algorithm)
	to the general setting proposed in this paper.
	In \cref{sec:minimaxub} we give a version of the algorithm presented here
	that is specialized to minimax games and refine the upper bound of this section,
	highlighting the advantages of the extra structure of minimax games.
	
	In this section we shall assume that the distributions $(P_i)_{i\in [L]}$ are subgaussian with a common parameter, which we take to be one for simplicity:
	\begin{assumption}[$1$-Subgaussian observations]
		\label{ass:subgauss}
		For any $i\in [L]$, $X\sim P_i(\cdot)$,
		\begin{align*}
		\sup_{\lambda\in \R} \EE{ \exp( \lambda(X-\E X) - \lambda^2/2) } \le 1\,.
		\end{align*}
	\end{assumption}
	\newcommand{\Et}{\E_t}
	\newcommand{\EEt}[1]{\Et\left[#1\right]}
	We will need a result for anytime confidence intervals for martingales with subgaussian increments.
	For stating this result, let $(\cF_t)_{t\in \N}$ be a filtration over the probability space $(\Omega,\cF,\P)$ holding our
	random variables and introduce $\Et[\cdot] = \EE{ \cdot| \cF_{t-1} }$.
	This result appears as (essentially) Theorem 8 in the paper by \citet{KauCaGa16} who also cite precursors:%
	\begin{lemma}[Anytime subgaussian concentration]
		\label{lem:anyconc}
		Let $(X_t)_{t\in \N}$ be an $(\cF_t)_{t\in \N}$-adapted $1$-subgaussian, martingale difference sequence (i.e.,
		for any $t\in \N$, $X_t$ is $\cF_t$-measurable, $\Et{X_t}=0$, and
		$\sup_{\lambda} \EEt{ \exp(\lambda(X_t)-\lambda^2/2) } \le 1$).
		For $t\in \N$,
		let $\overline{X}_t = (1/t)\,\sum_{s=1}^t X_s$, while for
		$t\in \N\cup \{0\}$ and
		$\delta\in [0,1]$ we let
		\begin{align*}
		\beta(t,\delta)
		= \log(1/\delta) + 3 \log \log(1/\delta) + (3/2) (\log(\log(et)))^+\,.
		\end{align*}
		Then, for any $\delta\in [0, 0.1]$,%
		\footnote{Note that $\beta(t,\delta)$ is also defined for $t=0$. The value used is arbitrary: It plays no role in the current result.
			The reason we define $\beta$ for $t=0$ is because it simplifies some subsequent definitions.}
		\begin{align*}
		\Prob{ \sup_{t\in \N} \frac{\overline{X}_t}{\sqrt{2\beta(t,\delta)/t}} > 1 } \le \delta\,.
		\end{align*}
	\end{lemma}
	
	For a fixed $i \in [L]$, let $N_t(i) = \sum_{s=1}^t \one{I_s = i}$
	denote the number of observations taken from $P_i(\cdot)$ up to time $t$.
	Define the confidence interval $[L^\delta_t(i), U^\delta_t(i)]$ for $\mu_i$ as follows:
	We let
	\[
	\hmu_t(i) = \frac{1}{N_t(i)}\sum_{s=1}^t \one{I_s=i}  Y_s\,,
	\]
	the empirical mean of observations from $P_i(\cdot)$ to be the center of the interval (when $N_t(i)=0$,
	we define $\hmu_t(i)=0$)
	and%
	\begin{align*}
	& L_t^\delta(i) = \max\left\{ L_{t-1}^\delta(i),\, \hmu_t(i) - \sqrt{\frac{2\beta(N_t(i),\delta/(2L))}{N_t(i)}}\right\}; \\
	& U_t^\delta(i) = \min\left\{U_{t-1}^\delta(i),\,  \hmu_t(i) + \sqrt{\frac{2\beta(N_t(i),\delta/(2L))}{N_t(i)}} \right\},
	\end{align*}
	where $\beta(t,\delta)$ is as in \cref{lem:anyconc} (note that when $N_t(i)=0$, the confidence interval is $(-\infty,+\infty)$).
	Let $T$ be the index of the round when the algorithm soon to be proposed stops (or $T=\infty$ if it does not stop).
	Let $\xi = \cap_{t\in [T],  i\in [L]} \{ \mu_i \in  [L^\delta_t(i), U^\delta_t(i)] \}$ be the ``good'' event when the
	proposed respective intervals before the algorithm stops \emph{all} contain $\mu_i$ for all $i\in [L]$.
	One can easily verify that, regardless the choice of the algorithm (i.e., the stopping time $T$),
	\begin{align}
	& \Prob{\xi}\ge 1-\delta\;, \label{eq:eventxi1} \\
	& \forall t\in \N, \,L^\delta_t(i) \le \hmu_t(i) \le  U^\delta_t(i). \label{eq:eventxi2}
	\end{align}
	For $S \subset \R^L$ define $f(S) = \{ f(s)\,:\, s\in S \}$. With this definition, if we let $S_t = \bigtimes_{i=1}^L [L_t^\delta(i),U_t^\delta(i)]$ then for any $j\in [K]$,
	$f_j(\mu) \in f_j(S_t)$ holds for any $t\ge 1$ on $\xi$.
	Thus, $f_j(S_t)$ is a valid, $(1-\delta)$-level confidence set for $f_j(\mu)$. For general $f$, these sets
	may have a complicated structure. Hence, we will adapt the following simplifying assumption:
	\begin{assumption}[Regular reward maps]
		\label{ass:ass1}
		The following hold:
		\begin{enumerate}[(i)]
			\item \label{ass:ass1:mon}
			The mapping function $f$ is monotonous
			with respect to the partial order of vectors: for any $u,v\in \R^L$, $u\le v$ implies $f(u) \le f(v)$;
			\item  \label{ass:ass1:cov}
			For any $u,v\in \R^L$, $u\le v$,
			$j\in [K]$, the set $D(j,u,v) = \{ i\in [L]\,: \, [f_j(u), f_j(v)]\subset [u_i, v_i] \,\}$ is non-empty.
		\end{enumerate}
	\end{assumption}
	\newcommand{\topone}{B}
	\newcommand{\toptwo}{C}
	\begin{wrapfigure}[9]{r}{0.5\textwidth}
		\vspace{-0.55cm}
		\begin{minipage}{.495\textwidth}
			\begin{algorithm}[H] 
				\caption{LUCB-micro}
				\label{alg:AlgGeneric}
				\begin{algorithmic}
					\FOR{$t=1, 2,\ldots$}
					\STATE Choose $B_t,C_t$ as in \cref{eq:bcchoice}
					\STATE Choose \emph{any} $(I_t,J_t)$ from $D_t(B_t)\times D_t(C_t)$%
					\STATE Observe $Y_{t,1} \sim P_{I_t}(\cdot)$, $Y_{t,2} \sim P_{J_t}(\cdot)$
					\STATE Update $[L^\delta_t(I_t), U^\delta_t(I_t)]$, $[L^\delta_t(J_t), U^\delta_t(J_t)]$
					\IF{Stop()}
					\STATE $J \leftarrow \topone_t$, $T \leftarrow t$
					\RETURN $(T,J)$
					\ENDIF
					\ENDFOR
				\end{algorithmic}
			\end{algorithm}
		\end{minipage}
	\end{wrapfigure}
	We will also let $D_t(j) = D(j,L^\delta_t,U^\delta_t)$.
	Note that the assumption is met when $f$ is the reward map underlying minimax games (see the next section).
	The second assumption could be
	replaced by the following weaker assumption without essentially changing our result:
	with some $a>0$, $b\in \R$, for any $j$, $u\le v$,
	$[f_j(u),f_j(v)] \subset [a u_i+b, a v_i + b]$ for some $i\in [L]$.
	The point of this assumption is that by guaranteeing that all intervals on the micro-observables shrink, the interval on the arm-rewards will also shrink at the same rate.
	We expect that other ways of weakening this assumption are also possible, perhaps at the
	price of slightly changing the algorithm (e.g., by allowing it to 
	use even more micro-observations per round).
	
	At time $t$, let
	\begin{align}
	\begin{split}
	\topone_t &= \argmax_{j\in[K]} f_j(L^\delta_t)\,, \quad\\
	\toptwo_t &= \argmax_{j\in[K], j\neq\topone_t} f_j(U^\delta_t)\,.
	\label{eq:bcchoice}
	\end{split}
	\end{align}
	($B$ stands for candidate ``\underline{b}est'' arm, $C$ stands for best ``\underline{c}ontender'' arm).
	Based on the above assumption, we can now propose our algorithm, LUCB-micro (cf. \cref{alg:AlgGeneric}).
	Following the idea of LUCB, LUCB-micro chooses $\topone_t$ and $\toptwo_t$ in an effort
	to separate the highest lower bound from the best competing upper bound.%
	\footnote{
		Using a lower bound departs from the choice of LUCB, which would use
		$f_j(\hat \mu_t)$ to define $\topone_t$.
		The reason of this departure is that we found it easier
		to work with a lower bound. We expect the two versions (original, our choice) to behave similarly.}
	To decrease the width
	of the confidence intervals, both for $\topone_t$ and $\toptwo_t$, a micro-observable is chosen with the help of
	\cref{ass:ass1}\eqref{ass:ass1:cov}.
	This can be seen as a generalization of the choice made in Maximin-LUCB
	by \citet{garivier2016maximin}. 
	Here, we found that the specific way Maximin-LUCB's choice is made 
	considerably obscured the idea behind this choice, which one can perhaps attribute
	to that the fact that the two-move setting makes it possible to write the choice in a more-or-less direct fashion.
	
	It remains to specify the `Stop()' function used by our algorithm. For this, we 
	propose the standard choice (as in LUCB):
	\begin{align}
	\label{eq:stoprule}
	\mbox{Stop() :} \quad f_{\topone_t}(L^\delta_t) \ge f_{\toptwo_t}(U^\delta_t).
	\end{align}
	\begin{quotation}
		\emph{All statements in this section assume that the assumptions stated so far in this section hold.}
	\end{quotation}
	
	The following proposition is immediate from the definition of the algorithm.
	\begin{restatable}[Correctness]{proposition}{propcorrectness}
		\label{prop:correctness}
		On the event $\xi$, LUCB-micro returns $J$ correctly: $J =j^*(\mu)$.
	\end{restatable}
	
	Let $T$ denote the round index when LUCB-micro stops%
	\footnote{The number of observations, or number of rounds as per \cref{fig:protocol},
		taken by LUCB-micro until it stops is $2T$.}
	and define
	$c= \frac{f_1(\mu) + f_2(\mu)}{2}$ and $\Delta =f_1(\mu) - f_2(\mu)$,  
	where we assumed that $f_1(\mu)>f_2(\mu)\ge \max_{j\ge 2} f_j(\mu)$.
	The main result of this section is a high-probability bound on $T$, which we present next.
	The following lemma is the key to the proof:
	\begin{restatable}{lemma}{lemlengthofInterval}
		\label{lem:lengthofInterval}
		Let $t<T$. Then, on $\xi$, there exists $J\in \{\topone_t, \toptwo_t\}$
		such that $c \in [f_J(L^\delta_t), f_J(U^\delta_t)]$ and $f_J(U^\delta_t) - f_J(L^\delta_t) \ge \Delta/2$.
	\end{restatable}
	The proof follows standard steps (e.g., \citealt{garivier2016maximin}). In particular,
	the above lemma implies that if $T>t$ then for $J\in\{ \topone_t, \toptwo_t \}$,
	$c \in [f_J(L^\delta_t), f_J(U^\delta_t)]$ and 
	$f_J(U^\delta_t) - f_J(L^\delta_t) \ge \Delta/2$.
	This in turn implies that for $i\in \{I_t,J_t\}$, $N_t(i)$ cannot be too large.
	\begin{restatable}[LUCB-micro upper bound]{theorem}{thmupperbound}
		\label{thm:upperbound}
		Let
		\begin{align*}
		H(\mu) &= \sum_{i \in [L]} \left\{\frac{1}{(c - \mu_i)^2} \bigwedge \frac{1}{(\Delta/2)^2}\right\}\,, \text{ and }\,\,
		t^*(\mu) = \min\{t\in \natu \,:\, 1+ 8 H(\mu)\beta(t,\delta/(2L)) \le t\}\,.
		\end{align*}
		Then, for $\delta\le 0.1$, on the event $\xi$, the stopping time $T$ of LUCB-micro
		satisfies $T \le t^*(\mu)$.
	\end{restatable}
	Note that $\beta(t,\delta) \propto \log \log t$ and thus $t^*(\mu)$ is well-defined.
	Furthermore, letting $c_\delta = \log(2L/\delta) + 3 \log \log(2L/\delta)$,
	for $\delta$ sufficiently small and $H(\mu)$ sufficiently large, elementary calculations give
	\begin{align*}
	t^*(\mu) \le 16 H(\mu) c_\delta + 16 H(\mu) \log \log( 8H(\mu) c_\delta )\,.
	\end{align*}
	\begin{remark}
		\label{rmk:upperboundGeneral}
		The constant $H(\mu)$ acts as a hardness measure of the problem.
		\cref{thm:upperbound} can be applied to the best arm identification problem in the multi-armed bandits setting, as it is a special case of our problem setup.
		Compared to state-of-the-art results available for this setting, our bound is looser in several ways: We lose on the constant factor multiplying $H(\mu)$
		\citep{kalyanakrishnan2012pac,jamieson2014best,jamieson2014lil,KauCaGa16},
		we also lose an additive term of $H(\mu)\log\log(L)$ \citep{chen2015optimal}.
		We also lose $\log(L)$ terms on the suboptimal arms \citep{SimJaRe17}.
		Comparing with the only result available in the two-move minimax tree setting, due to \citet{garivier2016maximin},
		our bound is looser than their Theorem~1.
		This motivates the refinement of this result to the minimax setting, which is done in the next section,
		and where we recover the mentioned result of \citet{garivier2016maximin}.
		On the positive side, our result is more generally applicable than any of the mentioned results.
		It remains an interesting sequence of challenges to prove an upper bound for this or some other algorithm
		which would match the mentioned state-of-the-art results, when the general setting is specialized.
	\end{remark}
	\section{Best move identification in minimax games}
	\label{sec:minimaxub}
	In this section we will show upper bounds on the number of observations LUCB-micro takes in the case of minimax game problems.
	We still assume that the micro-observations are subgaussian (\cref{ass:subgauss}) and the optimal arm is unique.
	To apply our result, this leaves us with showing that the payoff function in the minimax game satisfies the regularity assumption (\cref{ass:ass1}).
	
	Fix a minimax game structure $G = (M,H,p,\tau)$.
	We first show that Property~\eqref{ass:ass1:mon} of Assumption \ref{ass:ass1} holds. This follows easily from the following lemma,
	which can be proven by induction based on ``distance from the terminating states''.
	\begin{restatable}{lemma}{lemmonotonicpartialorder}
		\label{lem:monotonicpartialorder}
		For any $h\in H$ and $u,v \in [0,1]^L$ such that $u\le v$, $V(h,u) \le V(h, v)$.
	\end{restatable}
	From this result we immediately get the following corollary:
	\begin{corollary}\label{cor:mon}
		For $u,v \in [0,1]^L$ such that $u\le v$, $f(u) \le f(v)$, hence \cref{ass:ass1} \eqref{ass:ass1:mon}  holds.
	\end{corollary}
			\begin{wrapfigure}{r}{0.5\textwidth}
				\vspace*{-0.55cm}
				\begin{minipage}{.49\textwidth}
					\begin{algorithm}[H] 
						\caption{MinMax.}
						\label{alg:MinMax}
						\begin{algorithmic}
							\STATE {\bfseries Inputs:} $h\in H$, $u,v\in \R^L$.
							\IF{$h \in H_{\max}$}
							\RETURN $h$
							\ELSIF{$p(h) =-1$}
							\STATE{$h \gets \join(h,m(h,u))$}
							\ELSIF{ $p(h) =1$ }
							\STATE $h \gets \join(h, m(h,v))$
							\ENDIF
							\RETURN  $\MinMax(h,u,v)$
						\end{algorithmic}
					\end{algorithm}
				\end{minipage}
			\end{wrapfigure}
	For $j\in [K]$, $u,v\in \R^L$, $u\le v$,
	per Property \eqref{ass:ass1:cov} of \cref{ass:ass1},
	we need to show that the sets $D(j,u,v)$ are nonempty.
	For a history $h = (m_1, m_2, \ldots, m_\ell) \in H$ and $1\le k \le \ell$, we denote its length-$k$ prefix $(m_1, \dots,m_k)$ by  $h_k$.
	We give an algorithmic demonstration, which also shows how to efficiently pick an element of these sets.
	The resulting algorithm is called
	$\MinMax$ (cf. \cref{alg:MinMax}).
	We define $\MinMax$ in a recursive fashion:
	For each nonmaximal history the algorithm extends the history by adding the move which is optimal
	for $u$ for minimizing moves, while it extends it by adding the optimal move for $v$ for maximizing moves,
	and then it calls itself with the new history.
	The algorithm returns when its input is a maximal history.
	To show that
	$\tau(\MinMax(h,u,v)) \in D(j,u,v)$ we have the following result:
	\begin{restatable}{lemma}{lemSubsetInterval}
		\label{lem:SubsetInterval}
		Fix $u,v\in \R^L$, $u \le v$, and
		$j\in [K]$.
		Let $h = \MinMax( (j) )$ and in particular let
		$h = (m_1 = j, m_2,\dots,m_\ell)$. 
		Then, for all $1\le k < \ell$,
		\[
		[V(h_k, u) ,V(h_k, v)] \subset [V(h_{k+1}, u) ,V(h_{k+1}, v)]\,,
		\]
		where $h_k$ is the length-k prefix of $h$.
	\end{restatable}
	We immediately get that $i = \tau(h)$ is an element of $D(j,u,v)$:
	\begin{corollary}
		For $j,u,v,h$ as in the previous result, setting $i = \tau(h)$,
		$[f_j(v), f_j(v)]\subset [u_i, v_i)]$, hence $i\in D(j,u,v) \ne \emptyset$.
	\end{corollary}
	With this, we have shown that all the assumptions needed by \cref{thm:upperbound} are satisfied, and in particular,
	we can use $I_t = \MinMax( \topone_t, L_t^\delta, U_t^\delta)$ and $J_t = \MinMax( \toptwo_t, L_t^\delta, U_t^\delta)$.
	We call the resulting algorithm \LUCBMinMax. Then, \cref{thm:upperbound}  gives:
	\begin{corollary}
		\label{cor:upperboundMT}
		Let $\xi,c, \Delta$, $H(\mu), t^*$ be as in \cref{sec:ub}.
		If $T$ is the stopping time of  \LUCBMinMax running on a minimax game search problem
		then $T\le t^*$.
	\end{corollary}
	
	When applied to a minimax game, as defined in Section \ref{sec:prob}, the upper bound of \cref{cor:upperboundMT} is loose and can be further improved as shown in the result below.
	To state this result we need some further notation.
	Given a set of reals $S$, define the ``span'' of $S$ as  $\Span(S) = \max_{u,v\in S} u-v$.
	For a path $h\in H$ that connects some move in $[K]$ and some move in $[L]$: $h = (m_1,\dots,m_\ell)$ with some $\ell \ge 0$, $m_1 \in [K]$  and $m_\ell \in [L]$.
	\newcommand{\Vset}{\mathbb{V}}
	Finally, for $i\in [L]$ such that there is a unique path $h \in H$ satisfying $\tau(h) = i$, define $\Vset(i,\mu) = \{ V(h_k,\mu)\,: h =(m_1, \ldots, m_\ell = i),\, m_1\in [K],\, 1 \le k <\ell \}$. Let $\Vset(i,\mu)$ be an empty set if there is multiple $h\in H$ such that $\tau(h) = i$.
	\begin{restatable}[\LUCBMinMax on MinMax Trees]{theorem}{thmMTupperbound}
		\label{thm:MTupperbound}
		Let
		\begin{align*}
		H(\mu) = \sum_{i \in [L]} \min\{ \frac{1}{\Span(\Vset(i, \mu)\cup\{c, \mu_i\})^2}, \frac{4}{\Delta^2}\},
		\quad
		t^*(\mu) = \min\{t\in \natu \,:\,  1+ 8 H(\mu)\beta(t,\delta/(2L)) \le t\}\,.
		\end{align*}
		Then, on $\xi$, the stopping time $T$ of \LUCBMinMax satisfies $T\le t^*(\mu)$.
	\end{restatable}
	\begin{remark}
		\label{rmk:upperboundMinimax1}
		Note that this result recovers Theorem 1 of \citet{garivier2016maximin}. To see this note that for every leaf $(i,j)$ (as numbered in their paper), $\mu_{i,1} \in \Vset((i,j),\mu)$. Also note that $\mu_{i,1} \le \mu_{i,j}$, thus $|c - \mu_{i,j}| \le \max\{|c - \mu_{i,1}|, |\mu_{i,j} - \mu_{i,1}| \} $. Therefore, $\Span(\{\mu_{i,1}, \mu_{i,j}, c\}) = \max\{ |\mu_{i,1} - c|, |\mu_{i,j} - \mu_{i,1}| \}$.
	\end{remark}
	
	\section{Discussion and Conclusions}
	\label{sec:discussion}
	\paragraph{ The gap between the lower bound and the upper bound}
	There is a substantial gap between the lower and the upper bound. Besides the gaps that already exist in the multi-armed bandit setting
	and which have been mentioned before, there exists a substantial gap:
	In particular, it is not hard to show that in regular minimax game trees with a fixed branching factor of $\kappa$ and depth $d$,
	the upper bound scales with $O(\kappa^d)$ while the lower bound scales with $O(\kappa^{d/2})$. 
	One potential is to improve the lower bound so as to consider adversarial perturbations of the values assigned to the leaf nodes:
	That is, after the algorithm is fixed, an adversary can perturb the values of $\mu$ to maximize the lower bound.
	\citet{SimJaRe17} introduces an interesting technique for proving lower bounds of this form and they demonstrate nontrivial improvements
	in the multi-armed bandit setting.

	\begin{figure}
		\centering
		\begin{minipage}{\textwidth}
			\centering
			\includegraphics{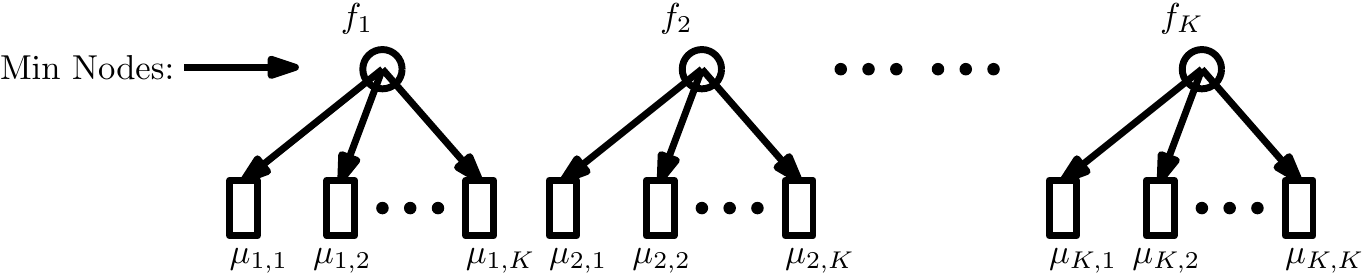}
			\caption{A 2-layer minimax tree.}
			\label{ex:example2}
		\end{minipage}\\
	\end{figure}
	\paragraph{Does the algorithm need to explore all leaves?}
	The hardness measure $H(\mu)$ is rooted in a uniform bound that suggests that all the leaves must potentially be pulled, which may not hold for some particular structure.
	In particular, the algorithm may be able to benefit from the specific structure of $f$, saving explorations on some leaves.
	We present one example when $f$ is a minimax game tree, as in Figure \ref{ex:example2}.
	Assume that $\mu_{1,i} = \mu^* >> \mu^{**} = \mu_{j,i}$ for $1\le i\le K$ and $2\le j\le K$.
	A reasonable algorithm would sample each arm once, then discover that the others arms are much less than the sampled leaf under arm 1. Then the algorithm will continue to explore the other leaves of arm 1, and decide arm 1 to be the best arm.
	This behavior is also in agreement with  our lower bound, where the resulting constraints are:
	\[
	N_{1,i} (\mu^* - \mu^{**})^2 \ge 2\log(1/4\delta); \quad  \sum_{i=1}^{K} N_{j,i} (\mu^* - \mu^{**})^2 \ge 2\log(1/4\delta)\quad \forall j\neq 1,
	\]
	which implies $N_{1,i}\ge 1$ and $\sum_{i=1}^{K} N_{j,i} \ge 1$ for $j\neq 1$ if $\mu^* - \mu^{**}$ is large enough. As we can see from this example, $(K-1)^2$ (out of $K^2$) leaves need no exploration at all. On the other hand, although we don't have a tight upper bound, our algorithm in practice manages to explore the remaining $K-1$ leaves under arm 1 for the next $K-1$ rounds, and then make the right decision.
	\begin{wrapfigure}{r}{0.5\textwidth}
		\begin{minipage}{0.495\textwidth}
			\includegraphics[height = 3.5cm]{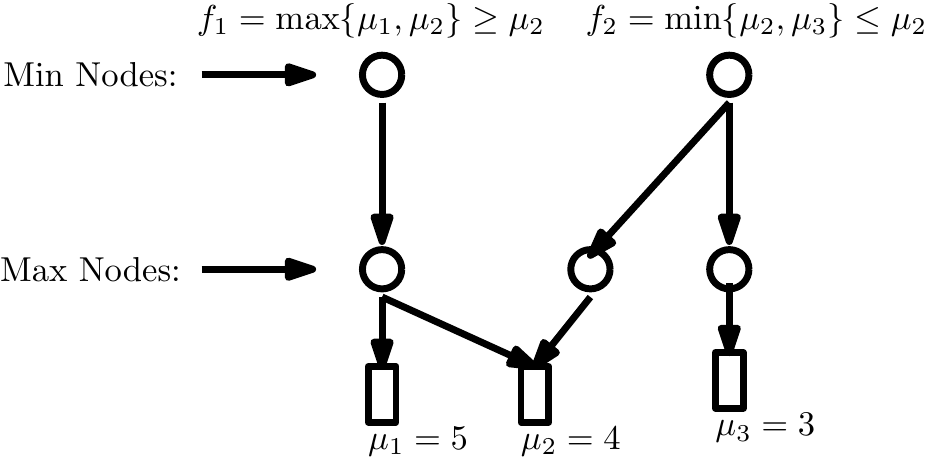}
			\caption{An example that needs no exploration.}
			\label{ex:example3}
		\end{minipage}
	\end{wrapfigure}
	In general, we would expect that a problem with a feed forward neural network structure is easier than that of a tree structure, as the share of the leaves provides more information and thus save the exploration.
	This is illustrated on \cref{ex:example3}, where an optimal arm can be identified solely based on the network structure,
	thus the algorithm requires {\bf 0} sample for all possible $\mu$.
	Note that our lower bound does not fail, as we have $D_{\mu} = \emptyset$ here.

	\begin{appendices}
		\section{The Uniqueness Assumption}
		Recall that throughout the paper, by definition, we have the following assumption:
		\begin{assumption}
			\label{ass:uniqueness}
			The instance is such that $j^*(\mu) = \argmax_j f_j(\mu)$ is unique.
		\end{assumption}
		We state this assumption explicitly here, so that we can refer to it easily throughout the appendix.
		\section{Proofs for \cref{sec:lb}}
		\label{sec:lbproof}
		Here we prove \cref{thm:LB}, which is restated for the convenience of the reader:
		\thmLB*
		
		We start with the two information theoretic results mentioned in the main body of the text.
		To state these results, let $\KL(P,Q)$ denote the Kullback--Leibler (KL) 
		divergence of two distributions $P$ and $Q$.
		Recall that this is $\int \log(\frac{dP}{dQ}) dQ$ when $P$ is absolutely continuous with respect to $Q$ and is infinite otherwise.
		For the next result, let $N(i) = \sum_{t=1}^T \one{I_t=i}$ denote the number of times an observation on the micro-observable with index $i\in [L]$ before time $T$.
		\begin{restatable}[Divergence decomposition]{lemma}{lemdivdec}
			\label{lem:divdec}
			For any $\mu,\mu'\in \R^L$ it holds that 
			\begin{align}
			\KL(\PmuA, \PmupA) = \frac12 \sum_{i=1}^L \EmuA[N(i)] \, (\mu_i-\mu'_i)^2\,.
			\label{eq:infoproc}
			\end{align}
		\end{restatable}
		Note that $\frac12(\mu_i-\mu'_i)^2$ on the right-hand side is the KL divergence between the normal distributions with means $\mu_i$ and $\mu_i'$ and both having a unit variance. The result, naturally, holds for other distributions, as well.
		This is the result that relies strongly on that we forced the same observations and same observation-choices for $t>T$.
		In particular, this is what makes the left-hand side of \eqref{eq:infoproc} finite! The proof is standard and hence is omitted.
		
		\begin{restatable}[High probability Pinsker, e.g., Lemma 2.6 from \citep{Tsy08:NonpBook}]{lemma}{lempinsker}
			\label{lem:pinsker}
			Let $P$ and $Q$ be probability measures on the same measurable space $(\Omega,\mathcal{F})$ and let $E\in \mathcal{F}$ be an arbitrary event.
			Then,
			\begin{equation*}
			P(E) + Q(E^{c})\geq\frac{1}{2}\exp(-\KL(P,Q))\,.
			\end{equation*}
		\end{restatable}
		\begin{proof}[of \cref{thm:LB}]
			WLOG, we may assume that $D_\mu^\circ$ is non-empty.
			Pick any $\Delta\in D_\mu^\circ$ and let $\mu' = \mu+\Delta$.
			Let $E = \{ J \ne 1\}$.
			Since $A$ is admissible, $\PmuA(E)\le \delta$.
			Further, since $\Delta\in D_\mu^\circ$, $1$ is not an optimal arm in $\mu'$.
			Hence, again by the admissibility of $A$, $\PmupA(E^c)\le \delta$.
			Therefore, by \cref{lem:pinsker},
			\begin{align*}
			2\delta \ge \PmuA(E)+\PmupA(E^c) &\ge \frac{1}{2}\exp( -\KL(\PmuA,\PmupA))\,.
			\end{align*}
			Now, plugging in \eqref{eq:infoproc} of \cref{lem:divdec} and reordering we get
			\begin{align*}
			\log(1/(4\delta)) \le \KL(\PmuA,\PmupA) = \frac12 \sum_{i=1}^L \EmuA[N(i)] \, (\mu_i-\mu'_i)^2\,.
			\end{align*}
			The result follows by continuity, after noting that $T = \sum_{i=1}^L N(i)$, that $\Delta\in D_\mu^\circ$ was arbitrary.
		\end{proof}
		
		\section{Proofs for \cref{sec:lb_minmax}}
		\newcommand{\Hprede}{H_{\text{prede}}}
		We start with the following lemma:
		\begin{lemma}
			\label{lem:helper1forProp4}
			Pick any $\mu \in \R^L$, $j\in [K]$. Then,
			\[
			\forall B\in \cB_j^+, f_j(\mu) \le \max\{ \mu_i\,:\, i\in B\};\quad \forall B\in \cB_j^-, f_j(\mu) \ge \min\{ \mu_i\,:\, i\in B\}.
			\] 
		\end{lemma}
		\begin{proof}
			Fix any $\mu\in \R^L$, $j\in [K]$, $B\in \cB_j^+$.
			Let $u = \max\{ \mu_i\,:\, i\in B\}$. We want to show that $f_j(\mu)\le u$.
			Define $\mu'\ge \mu$ such that $\mu'_i = \mu_i$ if $i\not\in B$, and $\mu_i' = u$ otherwise.
			As noted earlier (cf. \cref{cor:mon}), $f_j$ is monotonous. Hence, $f_j(\mu)\le f_j(\mu') \le u$, where
			the last inequality follows because $B\in \cB_j^+$.
			The proof concerning $\cB_j^-$ is analogous and is left to the reader.
		\end{proof}
		\lemMMLBRev*
		\begin{proof}
			Note that since $d\in D_\mu$, $f_1(\mu+d) \le f_j(\mu+d)$ for some $j\neq 1$. Fix such an index $j$. 
			
			To construct $B$ and $B'$, we will pick  $\tH_1\in \cH_1^+$, $\tH_j\in \cH_j^-$ and set 
			$B = \tau( \tH_1 \cap H_{\max})$ and $B' = \tau( \tH_j \cap H_{\max} )$.
			By the construction of $\cH_1^+$, to pick $\tH_1$ it suffices to specify the unique successor $h'$ in $\tH_1$ of any history $h\in \tH_1$ with $p(h)=-1$. For this, we let $h'\in H$ be the successor for which $V(h',\mu+d) = V(h,\mu+d)$. 
			Similarly, by the construction of $\cH_j^-$, to pick $\tH_j$ it suffices to specify the unique successor $h'$ in $\tH_j$ of any history $h\in \tH_j$ with $p(h)=+1$. Again, we let $h'\in H$ be the successor for which $V(h',\mu+d) = V(h,\mu+d)$. 
			Note that by \cref{prop:LB1}, $B\in \cB_1^+$ and $B'\in \cB_j^-$.
			
			Let us now turn to the proof of \eqref{lem:MMLBRev:1}.
			We start by showing that 
			\begin{align}
			\label{eq:f1maxeq}
			f_1(\mu+d) = \max \{ (\mu+d)_i\,:\, i\in B \}\,.
			\end{align}
			To show this, we first prove that 
			\begin{align}
			\label{eq:MMLBRev:1:ub}
			V(h,\mu+d) \le f_1(\mu+d)\, \quad \forall h\in\tH_1\,.
			\end{align}
			The proof uses induction based on the length of histories in $\tH_1$.
			
			There is only one history of length 1 (base case): $h = (1)$. 
			By the definition of $f_1$, $V(h,\mu+d) = f_1(\mu+d)$.
			Now, assume that the statement holds for all histories up to length $c\ge 1$.
			Take any $h\in \tH_1$ of length $c+1$. Let $h'\in \tH_1$ be the unique immediate predecessor of $h$: 
			$h\in \Hsucc(h')$. This is well-defined thanks to the definition of $H$ and the construction of $\tH_1$.
			If $p(h')=-1$ then, by the definition of $\tH_1$, $V(h,\mu+d) = V(h',\mu+d)$. By the induction hypothesis, 
			$V(h',\mu+d)\le f_1(\mu+d)$, implying $V(h,\mu+d)\le f_1(\mu+d)$.
			On the other hand, if $p(h') = +1$ then $f_1(\mu+d)\ge 
			V(h', \mu+d) = \max\{V(\tilde{h},\mu+d),\, \tilde{h}\in \Hsucc(h')\} \ge V(h,\mu+d)$, finishing the induction.
			Hence, we have proven \eqref{eq:MMLBRev:1:ub}. 
			
			Now, we claim that there exists $h^*\in \tH_1 \cap H_{\max}$ 
			such that $V(h^*,\mu+d) = f_1(\mu+d)$.
			This, together with \eqref{eq:MMLBRev:1:ub} implies \eqref{eq:f1maxeq}.
			
			We construct $h^*$ in a sequential process. 
			For this, we will choose a sequence of moves $m_1,\dots,m_k$ 
			such that $(m_1,\dots,m_i) \in H_{\max}\cap \tH_1$ and
			$V( (m_1,\dots,m_i), \mu+d ) = f_1(\mu+d)$ for any $1\le i \le k$.
			In a nutshell, this sequence is an ``optimal sequence of moves'' that starts with move 1, 
			which is also known as a principal variation for the game under move 1.
			In details, the construction is as follows:
			To start, we choose $m_1=1$.
			Then $V( (m_1),\mu+d) = f_1(\mu+d)$, by the definition of $f_1$.
			Assume that for some $i\ge 1$, we already chose $(m_1,\dots,m_i)$ so that 
			$V( (m_1,\dots,m_i), \mu+d ) = f_1(\mu+d)$  holds.
			If $h\defeq (m_1,\dots,m_i)\in H_{\max}$, we let $k=i$ and we are done. 
			Otherwise, let $m_{i+1} = m(h, \mu+d)$ (this is the ``optimal move'' at $h$ under valuation $\mu+d$).
			Thus, $V( \join(h,m_{i+1}), \mu+d ) = V(h,\mu+d) = f_1(\mu)$. Further, by the construction of $\tH_1$,
			$\join(h,m_{i+1})\in \tH_1$. Since all histories in $H$ are bounded in length, the process ends after some
			$k$ moves for some finite $k$, at which point we are done proving our statement.
			
			To recap, so far we have proved \eqref{eq:f1maxeq}. An entirely analogous proof (left to the reader) 
			shows that also $f_j(\mu+d) = \min\{ (\mu+d)_i\,:\, i\in B'\}$.
			
			We now prove that $f_1(\mu+d) = f_j(\mu+d)$, finishing the proof of \eqref{lem:MMLBRev:1}.
			Assume to the contrary that $f_1(\mu+d)<f_j(\mu+d)$.
			Consider the map $g: \alpha \mapsto f_1(\mu+\alpha d) - f_j(\mu+\alpha d)$ on the interval $\alpha \in [0,1]$.
			Note that $g$ is continuous, $g(0)>0>g(1)$. Hence, by the intermediate value theorem, there exists $\alpha\in (0,1)$ such that $g(\alpha)=0$. Note that $f_1(\mu+\alpha d) = f_j(\mu+\alpha d)$. Hence, $\alpha d \in D_\mu$.
			Since $\alpha |d| < |d|$, $d\in D_{\mu}^{\min}$ cannot hold, a contradiction. Hence, $f_1(\mu+d) = f_j(\mu+d)$.

			Let us now turn to the proof of \eqref{lem:MMLBRev:2}.
			We prove that $d_i \le 0$ holds for all $i\in B\setminus B'$. (The statement concerning elements of $B'\setminus B$
			follows similarly, the details are left to the reader.)
			For the proof, assume to the contrary of the desired statement 
			that there exists some $i\in B\setminus B'$ such that $d_i > 0$. 
			Let $d'\in \R^L$ be such that $d'_k = d_k$ for $j\neq i$, and $d'_i = 0$. 
			Thus, $d'<d$. 
			By \cref{cor:mon}, 
			$f_1(\mu+d')\le f_1(\mu+d)\le  f_j(\mu+d) 
			= \min\{(\mu+d)_k\,:\, k\in B'\} 
			= \min\{(\mu+d')_k\,:\, k\in B'\} \le f_j(\mu+d')$, 
			where the last equality is due to $i\not\in B'$ (hence, $(\mu+d)|_{B'} = (\mu+d')|_{B'}$)
			while the last inequality follows from \cref{lem:helper1forProp4}.
			This implies that $d'\in D_{\mu}$. This together with $|d'|<|d|$ contradicts $d\in D_{\mu}^{\min}$.
			Thus, \eqref{lem:MMLBRev:2} holds.
			
			It remains to prove \eqref{lem:MMLBRev:3}. For this pick $i\in B$. 
			Since $f_1(\mu+d) = f_j(\mu+d)$ has already been established, 
			it suffices to show that  either $(\mu+d)_i = f_1(\mu+d)$ or $d_i = 0$. 
			(The case when $i\in B'$ is symmetric and is left to the reader.)
			If $i\in B\cap B'$ then by \eqref{lem:MMLBRev:1}, $(\mu+d)_i \le \max_{k\in B}  (\mu+d)_k=
			f_1(\mu+d) = f_j(\mu+d) = \min_{k\in B'} (\mu+d)_k \le (\mu+d)_i$, showing that 
			$(\mu+d)_i = f_1(\mu+d) = f_j(\mu+d)$.
			Hence, assume that $i\not\in B\cap B'$. 
			If $d_i=0$ or $(\mu+d)_i = f_1(\mu+d)$ then we are done.
			Otherwise, by \eqref{lem:MMLBRev:2}, $d_i< 0$ and by
			\eqref{lem:MMLBRev:1}, $(\mu+d)_i<\max_{k\in B} (\mu+d)_k = f_1(\mu+d)$.
			Let $\epsilon = f_1(\mu+d) - (\mu+d)_i$. Note that $\epsilon>0$.
			Define $d'\in \R^L$ so that $d'_k = d_k$ if $k\ne i$ and let $d'_i = -(d_i+\epsilon)_{-}$. That is, $d_i$ is shifted up towards
			zero by a positive amount so that it never crosses zero.
			Then, $|d'| < |d|$. Note also that $\mu_i + d_i' = \mu_i + \min(d_i+\epsilon,0)
			\le \mu_i + d_i + \epsilon = f_1(\mu+d) = \max_{k\in B} (\mu+d)_k$. Hence, $\max_{k\in B} (\mu+d')_k = \max_{k\in B} (\mu+d)_k = f_1(\mu+d)$ and thus by \cref{lem:helper1forProp4},
			$f_1(\mu+d') 
			\le \max_{k\in B} (\mu + d)_k' 
			= f_1(\mu+d)$. By \eqref{lem:MMLBRev:1}, $f_1(\mu+d) = f_j(\mu+d) = \min_{k\in B'} (\mu+d)_k$.
			By the definition of $d'$ (thanks to $i\not\in B'$) and \cref{lem:helper1forProp4},
			$\min_{k\in B'} (\mu+d)_k = \min_{k\in B'} (\mu+d')_k \le f_j(\mu+d')$.
			Putting together the inequalities, we get $f_1(\mu+d')\le f_j(\mu+d')$. Hence, $d'\in D_\mu$.
			However, this and $|d'|<|d|$ contradict $d\in D_\mu^{\min}$, finishing the proof of \eqref{lem:MMLBRev:3}.
		\end{proof}
		\newcommand{\cI}{\mathcal{I}}
		\newcommand{\rel}[2]{\stackrel{\text{\footnotesize (#2)}}{#1}}
		\begin{lemma}
			\label{lem:decreasinggraph}
			Given any $\mu\in \R^L$ and any $\theta \in \R$, define $\mu'$ as follows:
			\begin{align*}
			\mu_i' = 
			\begin{cases}
			\theta, \quad i\in \cI;\\
			\mu_i, \quad \text{otherwise,}
			\end{cases}
			\end{align*}
			where $\cI\subset\{i\,:\, \mu_i \ge \theta\}$.
			Then, $f_j(\mu') \ge \min\{\,\theta, f_j(\mu)\,\}$ for any $j\in [K]$.
		\end{lemma}
		\begin{proof}
			Fix $j\in [K]$.
			We prove $V(h, \mu') \ge \min\{\theta,\, V(h, \mu)\}$ for $h\in H$ by induction based on how close a history $h$ is to being a maximal history. Note that this suffices to prove the statement thanks to 
			$f_j(\mu') = V((j), \mu') \ge \min\{\theta,\, V((j), \mu)\} = \min\{ \theta,\, f_j(\mu)\}$.
			
			Define function $c$ so that $c(h)=0$ if $h\in H_{\max}$, and $c(h) = 1 + \max\{c(h')\,:\, h'\in \Hsucc(h) \}$ otherwise.\\
			{\it Base case:} If $h\in H_{\max}$, then $V(h, \mu') = \mu_i' \in \{\mu_i,\theta\} \ge \min\{\theta, \mu_i \} = \min\{\theta, V(h,\mu) \}$ for some $i\in [L]$.\\
			{\it Induction step:} Assume that for any $h\in H$ such that $ c(h)\le c$, $V(h, \mu') \ge \min\{\theta,\, V(h, \mu)\}$.
			Given $h$ such that $c(h)=c+1$, if $p(h)=1$, 
			\begin{align*}
			V(h, \mu') & = \max\{V(h', \mu')\,:\, h'\in \Hsucc(h) \}\\
			& \ge \max\{ \min\{\theta, V(h', \mu)\} \,:\, h'\in \Hsucc(h)\} \quad \text{(by induction)}\\
			&  \rel{\ge}{a}  \min\{ \theta,\, V(h'^*, \mu)\}\\
			& = \min\{ \theta,\, V(h,\mu)\},
			\end{align*}
			where in (a), $h'^*$ is the optimal $h'$ such that $ V(h'^*, \mu) = V(h,\mu)$. If $p(h)=-1$, 
			\begin{align*}
			V(h, \mu') & = \min\{V(h', \mu')\,:\, h'\in \Hsucc(h) \}\\
			& \ge \min\{ \min \{\theta,\, V(h', \mu) \}\,:\, h'\in \Hsucc(h)\}\\
			& \rel{\ge}{b} \min\{ \theta,\, \min\{ V(h', \mu) \,:\, h'\in \Hsucc(h)\}\}\\
			& \ge \min\{\theta,  V(h, \mu)\}.
			\end{align*}
			Here (b) holds because for any $h'\in \Hsucc(h)$, $V(h', \mu) \ge\min\{ V(h', \mu) \,:\, h'\in \Hsucc(h)\} $, thus $\min \{\theta\,,\, V(h', \mu) \} \ge \min\{ \theta,\, \min\{ V(h', \mu) \,:\, h'\in \Hsucc(h)\}\}$.
		\end{proof}
		With this, we are ready to prove \cref{prop:MiniSigDepart}, which we repeat here for the reader's convenience:
		\propMiniSigDepart*
		\begin{proof}
			First we prove $D_\mu^{\min} \subset S$. 
			For this take any $d\in D_\mu^{\min}$.
			Since $d\in D_\mu$, 
			by \cref{lem:MMLBRev}, for some $j>1$,
			$f_1(\mu+d) = f_j(\mu+d)$.
			WLOG assume $j=2$.
			We will  prove that:
			\begin{align}
			\label{eq:lemma5}
			\exists B\in\cB_1^+,\, B'\in \cB_2^-\mbox{ s.t. } \forall i\in (B\cup B')^c,\, d_i = 0.
			\end{align}
			By \cref{lem:MMLBRev},
			there exist $B\in \cB_1^+$ and $B' \in \cB_2^-$ such that 
			\[\max\{(\mu+d)_i\,:\, i\in B\} = f_1(\mu+d) = f_2(\mu+d) = \min\{(\mu+d)_i\,:\, i\in B'\};
			\]
			Take these sets and pick some $i\in (B\cup B')^c$. If $d_i=0$, we are done. 
			Otherwise, let $d'_k = d_k$ for all $k\neq i$ and let $d'_i=0$. 
			Then, $|d'|<|d|$.
			By \cref{lem:helper1forProp4} and \cref{lem:MMLBRev} \eqref{lem:MMLBRev:1}, 
			\begin{align*}
			f_1(\mu+d') & \le \max\{(\mu+d')_j\,:\, j\in B \} = \max\{(\mu+d)_j\,:\, j\in B \} \\
			& = f_1(\mu+d) \\
			& \le f_2(\mu+d) \\
			& = \min\{(\mu+d)_j,\, j\in B' \} = \min\{(\mu+d')_j\,:\, j\in B' \}\\
			& \le f_2(\mu+d').
			\end{align*}
			Thus $d'\in D_{\mu}$, which contradicts that $d\in D_{\mu}^{\min}$, establishing \eqref{eq:lemma5}. Also we have $ f_1(\mu+d) = f_2(\mu+d)$.
			
			Let $\theta = f_1(\mu+d) = f_2(\mu+d)$. For $i\in B\setminus B'$, by \cref{lem:MMLBRev} \eqref{lem:MMLBRev:2} and \eqref{lem:MMLBRev:3}, $d_i = -(\mu_i-\theta)_+$. 
			Similarly, $d_i = (\mu_i-\theta)_-$ for $i\in B'\setminus B$. 
			Note that for $i\in B\cap B'$, 
			\begin{align*}
			(\mu+d)_i & \le \max\{(\mu+d)_i\,:\, i\in B\cap B'\}\\
			& \le\max\{(\mu+d)_i\,:\, i\in B\}  \\
			& = f_1(\mu+d) = \theta = f_2(\mu+d)\\
			& = \min\{(\mu+d)_i\,:\, i\in B'\} \\
			& \le \min\{(\mu+d)_i\,:\, i\in B'\cap B\}\\
			& \le (\mu+d)_i\,.
			\end{align*}
			Thus, $(\mu+d)_i = \theta$, and therefore $d_i = \theta - \mu_i$.
			It remains to prove $\theta \in [f_2(\mu),f_1(\mu)]$.
			We prove this by contradiction. 
			Assume that $\theta < f_2(\mu)$. Define $d'$ as follows:
			\[
			d_i' = 
			\begin{cases}
			-(\mu_i - f_2(\mu))_+\,, & \text{if } i\in B\,;\\
			0\,, & \text{otherwise.}
			\end{cases}
			\]
			We will prove the following claims: 
			\begin{enumerate}[(i)]
				\item \label{propMiniSigDepart:c1} $|d'|< |d|$;
				\item \label{propMiniSigDepart:c2} $f_1(\mu+d') \le f_2(\mu)$;
				\item \label{propMiniSigDepart:c3} $f_2(\mu+d') \ge f_2(\mu)$.
			\end{enumerate}
			Altogether these contradict $d\in D_{\mu}^{\min}$. 
			
			To show \eqref{propMiniSigDepart:c1},
			note that for $i\in B^c$ or $i\in B$ such that $\mu_i\le f_2(\mu)$, $|d_i'|=0\le |d_i|$. 
			Assume $i\in B$ such that $\mu_i > f_2(\mu)>\theta$. 
			Then $0> d_i' = f_2(\mu) - \mu_i > \theta - \mu_i = - (\mu_i - \theta)_+ = d_i$, thus $|d_i'| < |d_i|$. 
			Therefore $|d'|< |d|$, proving \eqref{propMiniSigDepart:c1}. 
			
			For  \eqref{propMiniSigDepart:c2}, note that for $i\in B$, $\mu_i+d_i' = \mu_i - (\mu_i - f_2(\mu))_+ \le f_2(\mu)$,
			thus $\max_{i\in B} \mu_i + d_i' \le f_2(\mu)$.
			By \cref{lem:helper1forProp4}, we also have $f_1(\mu+d') \le \max_{\in B} \mu_+d'_i$, which together with the previous inequality implies \eqref{propMiniSigDepart:c2}.
			
			Lastly, for proving \eqref{propMiniSigDepart:c3} 
			define $\cI = \{i\in B\,:\, \mu_i\ge f_2(\mu)\}$. Then $\mu' := \mu+d'$ can be rewritten as 
			\begin{align*}
			\mu_i' = 
			\begin{cases}
			f_2(\mu)\,, & \text{if } i \in \cI\,;\\
			\mu_i\,, & \text{otherwise.}
			\end{cases}
			\end{align*}
			By \cref{lem:decreasinggraph}, $f_2(\mu+d') \ge f_2(\mu)$, showing \eqref{propMiniSigDepart:c3} . 
			
			The inequality $\theta \le f_1(\mu)$ can also be proved using analogous ideas. 
			Therefore, $\theta \in [f_2(\mu), f_1(\mu)]$. 
			Combining all the previous statements leads to the conclusion $D_{\mu}^{\min}\subset S$.
			
			Let us now prove that $S\subset D_\mu$.
			Take any element $\Delta \in S$.
			Let $j\in [K]$, $B\in \cB_1^+$ and $B'\in \cB_j^-$ as in the definition of $S$.
			WLOG assume that $j=2$.
			Let $\mu' = \mu+\Delta$.
			It suffices to show that $f_1(\mu')\le \theta$ and $f_2(\mu')\ge \theta$.
			We show $f_1(\mu')\le \theta$, leaving the proof of the other relationship to the reader (the proof is entirely analogous
			to the one presented).
			By \cref{lem:helper1forProp4}, it suffices to show that $\max\{ \mu_i'\,:\, i\in B \} \le \theta$.
			When $i\in B \setminus B'$, $\Delta_i = -(\mu_i - \theta)_+$. 
			Thus, $\mu_i' = \mu_i - \max( \mu_i-\theta,0) = \mu_i +\min(\theta-\mu_i,0) \le \mu_i + \theta-\mu_i \le \theta$.
			If $i \in B \cap B'$, $\mu_i' = \mu_i + (\theta-\mu_i) = \theta$, thus finishing the proof.
		\end{proof}
		
		\section{Proofs for \cref{sec:ub}}
		We start with the correctness result:
		\propcorrectness*
		\begin{proof}
			Assume to the contrary that $J\ne j^*(\mu)$.
			WLOG let $j^*(\mu)=1$.
			By \cref{ass:ass1}\eqref{ass:ass1:mon}
			the definition of $\xi$ and that of $J$, $C_T$, the stopping rule,
			$f_J(\mu) \ge f_J(L_T^\delta) \ge f_{\toptwo_T}(U_T^\delta) \ge f_1(U_T^\delta) \ge f_1(\mu)$.
			This contradicts~\cref{ass:uniqueness}.
		\end{proof}
		
		For proving the sample complexity bound, we consider the following result:
		\lemlengthofInterval*
		\begin{proof}
			We first prove that
			$c \in \cI \defeq \cup_{j\in \{\topone_t,\toptwo_t\}} [f_j(L^\delta_t), f_j(U^\delta_t)]$.
			For this, it suffices to show that
			it does not hold that $c\in \cI^c$ where $\cI^c = \R\setminus \cI$ is the complementer of $\cI$.
			Now, $c\in \cI^c$ holds iff at least one of the four conditions hold:
			\emph{(i)} $f_{\topone_t}(L^\delta_t)>c$ and $f_{\toptwo_t}(L^\delta_t)>c$;
			\emph{(ii)} $f_{\topone_t}(U^\delta_t)<c$ and $f_{\toptwo_t}(U^\delta_t)<c$;
			\emph{(iii)} $f_{\topone_t}(U^\delta_t)<c$ and $f_{\toptwo_t}(L^\delta_t)>c$;
			\emph{(iv)} $f_{\topone_t}(L^\delta_t)>c$ and $f_{\toptwo_t}(U^\delta_t)<c$.
			Consider the following:
			\begin{enumerate}[\mbox{Case} (i)]\itemsep0em
				\item implies that $f_{\topone_t}(\mu) \ge f_{\topone_t}(L^\delta_t)>c$ and similarly $f_{\toptwo_t}(\mu)>c$. Thus there are two arms with payoff greater than $c$, which contradicts \cref{ass:uniqueness}.
				\item implies that no arm has payoff above $c$, which contradicts the definition of $c$.
				\item Then $f_{\toptwo_t}(L^\delta_t)>c>f_{\topone_t}(U^\delta_t)\ge f_{\topone_t}(L^\delta_t)$,
				which contradicts the definition of $\topone_t$.
				\item If this is true, then by definition the algorithm has stopped, hence $t\not<T$.
			\end{enumerate}
			Thus, we see that $c\in \cI^c$ cannot hold and hence $c\in [f_J(L^\delta_t), f_J(U^\delta_t)]$ for either $J=\topone_t$ or $J = \toptwo_t$,
			proving the first part.
			Next, note that for any $j\in [L]$, $|c - f_j(\mu)| \ge \frac{\Delta}{2}$.
			Hence, also $|c-f_J(\mu)|\ge \frac{\Delta}{2}$.
			Also note that $f_J(\mu) \in [f_J(L^\delta_t), f_J(U^\delta_t)]$. Thus,
			$f_J(U^\delta_t) - f_J(L^\delta_t) \ge |c - f_J(\mu)| \ge \frac{\Delta}{2}$.
		\end{proof}
		We can now prove \cref{thm:upperbound}:
		\thmupperbound*
		\begin{proof}
			Let $\tau$ be a fixed deterministic integer.
			Now, on $\xi$,			
			\begin{align*}
			& \quad \min(T, \tau) \le 1+ \sum_{t=1}^{\tau} \one{t < T} \\
			& \rel{\le}{a} 1+
			\sum_{t=1}^{\tau} \one{\exists J\in\{ \topone_t, \toptwo_t \} \text{ s.t. } c \in [f_J(L^\delta_t), f_J(U^\delta_t)] \text{ and } f_J(U^\delta_t) - f_J(L^\delta_t) \ge \Delta/2} \\
			& \rel{\le}{b} 1+
			\sum_{t=1}^{\tau} \one{\exists I\in\{I_t, J_t\} \text{ s.t. } c \in [L^\delta_t(I), U^\delta_t(I)] \text{ and } U^\delta_t(I) - L^\delta_t(I) \ge \Delta/2} \\
			& \le 1+ \sum_{t=1}^{\tau} \sum_{i \in [L]} \one{i\in\{I_t, J_t\}}\one{c \in [L^\delta_t(i), U^\delta_t(i)] \text{ and } U^\delta_t(i) - L^\delta_t(i) \ge \Delta/2} \\
			& \rel{\le}{c} 1+
			\sum_{t=1}^{\tau} \sum_{i \in [L]} \one{i\in\{I_t, J_t\}}\one{N_t(i)\le 8\beta(N_t(i),\delta/(2L)) \left(\frac{1}{(c - \mu_i)^2} \wedge \frac{1}{(\Delta/2)^2}\right) }\\
			& \rel{\le}{d} 1+
			\sum_{i \in [L]}  \sum_{t=1}^{\tau} \one{i\in\{I_t, J_t\}}\one{N_t(i)\le 8\beta(\tau,\delta/(2L)) \left(\frac{1}{(c - \mu_i)^2} \wedge \frac{1}{(\Delta/2)^2}\right) }\\
			& \le 1+ \sum_{i \in [L]} 8\beta(\tau,\delta/(2L)) \left(\frac{1}{(c - \mu_i)^2} \wedge \frac{1}{(\Delta/2)^2}\right)\\
			& = 1+ 8 H(\mu)\beta(\tau,\delta/(2L))\,.
			\end{align*}
			Here,
			(a) holds by the first part of \cref{lem:lengthofInterval},
			(b) holds by \cref{ass:ass1}\eqref{ass:ass1:cov},
			(c) holds by the definition of $\beta$,
			(d) holds because $\beta(\cdot,\delta/(2L))$ is increasing.
			Picking any $\tau$ such that $8H(\mu)\beta(\tau,\delta/(2L)) \le \tau-1$, we have
			$\min(T,\tau) \le \tau$, showing that $T \le \min(T,\tau) \le \tau$.
		\end{proof}
		\section{Proofs for \cref{sec:minimaxub}}
		\lemmonotonicpartialorder*
		\begin{proof}
			We prove the result by induction based on how close a history $h$ is to being a maximal history.
			As in an earlier proof, for $h\in H$, we let $c(h) = 0$ if $h\in H_{\max}$ and otherwise we let
			$c(h) = 1+\max\{ c(h')\,: h' \in \Hsucc(h) \}$,
			where
			recall that $\Hsucc(h)$ denotes the set of immediate successors of $h\in H$ in $H$.

			{\it Base case: } If $c(h)=0$ (i.e., $h\in H_{\max}$), then $V(h,u) = u_{\tau(h)} \le v_{\tau(h)} = V(h,v)$.
			
			{\it Induction step: } Assuming that for all the $h' \in H$ with $c(h)\le c$ with some $c\ge 0$
			it holds that $V(h',u) \le V(h',v)$.
			Take $h\in H$ such that $c(h)=c+1$. WLOG assume that $p(h)=1$.
			We have:
			\begin{align*}
			V(h,u) & = V(\join(h,m(h,u)),u)
			\le V(\join(h,m(h,u)),v)
			\le V(\join(h,m(h,v)),v) = V(h,v)\,,
			\end{align*}
			where the first and the last equalities are by definition, the first inequality is by the induction hypothesis, and the second inequality is due to the definition of $m(h,v)$.
		\end{proof}
		
		\lemSubsetInterval*
		\begin{proof}
			Fix  $0\le k<\ell$ and $u\le v$.
			WLOG assume that $p(k)=1$.
			By the definition of $V(h,\mu)$ and $m(h,\mu)$,
			\begin{align*}
			V(h,v)
			= \max\{ V( h', v) \,:\, h'\in \Hsucc(h) \}
			= V(\join(h,m(h,v)),v) \,.
			\end{align*}
			Hence, by the definition of $\MinMax$ and the above identity,
			$V(h_k, v) = V(h_{k+1}, v)$. 
			Further, $V(h_k, u)
			= \max\{ V( h', u) \,:\, h'\in \Hsucc(h) \}
			\ge V(h_{k+1}, u)$. Thus,
			\[V(h_k, v) \le V(h_{k+1}, v) \mbox{ and }  V(h_k, u) \ge V(h_{k+1}, u)\,,\]
			finishing the proof.
		\end{proof}
		
		\thmMTupperbound*
		\begin{proof}
			Recall that $I_t =\tau(\MinMax(\topone_t,L_t^\delta,U_t^\delta))$ and $J_t= \tau(\MinMax(\toptwo_t,L_t^\delta,U_t^\delta))\}$.
			Assume that $\xi$ holds.
			We prove that
			$\Vset(I_t,\mu) \subset [L_t^\delta(I_t), U_t^\delta(I_t)]$
			and
			$\Vset(J_t,\mu) \subset [L_t^\delta(J_t), U_t^\delta(J_t)]$ hold.
			The rest of the proof is similar to that of \cref{thm:upperbound}.
			
			Consider $I_t$. The proof for $J_t$ works the same way and is hence omitted.
			If there is multiple path $h\in H$ such that $\tau(h) = I_t$, then $\Vset(I_t, \mu) = \emptyset \subset [L_t^\delta(I_t), U_t^\delta(I_t)]$.
			Otherwise, let $h\in H$ be the unique path. Since $I_t$ is pulled, $h = \MinMax(m)$ for some $m\in M$.
			Note that \cref{lem:SubsetInterval} implies that
			$ [V(h_k, L_t^\delta) ,V(h_k, U_t^\delta)] \subset [V(h_{k+1}, L_t^\delta) ,V(h_{k+1}, U_t^\delta)]$.
			Thus it is sufficient to prove that for $1\le k<\ell$, $V(h_k, \mu) \in [V(h_{k}, L_t^\delta) ,V(h_{k}, U_t^\delta)]$.
			However, this follows by \cref{lem:monotonicpartialorder} and because
			on the event $\xi$, $ L_t^\delta \le \mu \le U_t^\delta$ holds.
			
			Now let $S(i)  =\Vset(i, \mu)\cup\{c, \mu_i\}$.
			Fix $t<T$.
			By the above result and by \cref{lem:lengthofInterval}, for one of $J = \topone_t$ or $J = \toptwo_t$,
			if $I = \MinMax(J,L_t^\delta,U_t^\delta)$
			then $S(I)\subset [L_t^\delta(I), U_t^\delta(I)]$,
			which implies that $U_t^\delta(I) - L_t^\delta(I) \ge \Span(S(I))$. Therefore, 
			\begin{align*}
			\min(T, \tau) & \le 1+  \sum_{t=1}^{\tau} \one{t < T} \\
			& \le 1+ \sum_{t=1}^{\tau} \one{\exists I\in\{I_t, J_t\} \text{ s.t. } U_t^\delta(I) - L_t^\delta(I) \ge \Span(S(I))} \\
			& \le 1+ \sum_{t=1}^{\tau} \sum_{i \in [L]} \one{i\in\{I_t, J_t\}}\one{N_t(i)\le \frac{8\beta(N_t(i),\delta/(2L))}{\Span(S(i))^2} }\\
			& \le 1+ \sum_{i \in [L]}  \sum_{t=1}^{\tau}\one{i\in\{I_t, J_t\}}\one{N_t(i)\le \frac{8\beta(\tau,\delta/(2L))}{\Span(S(i))^2} }\\
			& \le 1+ \sum_{i \in [L]} \frac{8\beta(\tau,\delta/(2L))}{\Span(S(i))^2} = 1+ 8H(\mu)\beta(\tau,\delta/(2L))
			\end{align*}
		\end{proof}
	\end{appendices}

	
	\vskip 0.2in
	\bibliography{main}
\end{document}